\documentclass[11pt,a4paper]{amsart}
\usepackage{amsmath,amssymb,amsthm,amsfonts}
\usepackage{graphicx}
\usepackage{hyperref}
\usepackage{enumitem}
\usepackage{booktabs}
\usepackage{tikz}
\usepackage{tabularx}
\usetikzlibrary{arrows.meta, positioning, calc}
\usepackage[margin=1in]{geometry}
\usetikzlibrary{decorations.pathreplacing}

\newtheorem{theorem}{Theorem}[section]
\newtheorem{lemma}[theorem]{Lemma}
\newtheorem{proposition}[theorem]{Proposition}

\newtheorem{remark}[theorem]{Remark}
\newtheorem{definition}[theorem]{Definition}

\theoremstyle{definition}
\newtheorem{assumption}{Assumption}

\newcommand{\cG}{\mathcal{G}}

\newcommand{\bbE}{\mathbb{E}}
\newcommand{\bbR}{\mathbb{R}}
\newcommand{\KL}{\mathrm{KL}}
\newcommand{\Iobs}{I_{\mathrm{obs}}}
\newcommand{\Icom}{I_{\mathrm{com}}}
\newcommand{\Imis}{I_{\mathrm{mis}}}
\newcommand{\lmin}{\lambda_{\min}}
\newcommand{\lmax}{\lambda_{\max}}
\newcommand{\thetastar}{\theta^*}

\title{Expectation-Maximization as a Spectrally Governed Relaxation Flow}
\author{Qiao Wang}
\address{School of Information Science and Engineering, and School of Economics and Management, Southeast University, 211189, China}
\email{qiaowang@seu.edu.cn}
\date{\today}

\begin{document}

\subjclass[2020]{Primary 62F10, 94A15; Secondary 37C25, 47B25, 53B12, 65K10, 65B99}

\keywords{EM algorithm, relaxation operator, missing information principle,
spectral gap, information-geometric Hessian, Geo-Adaptive accelerator}

\maketitle


\begin{abstract}
The expectation--maximization (EM) algorithm combines global
monotonicity, local linear convergence, and strong practical
robustness, but these features are usually analyzed separately.
Global descent is nonlinear, whereas local convergence is governed
by the spectrum of the linearized EM map.
How these two levels fit into a single dynamical picture has
remained less transparent.

We make explicit the latent-variable operator that connects them.
Along the EM trajectory, the likelihood increment admits a global
energy decomposition in terms of posterior-relative entropy.
Linearization at a nondegenerate maximizer $\theta^\ast$ reveals
the local operator
\[
\mathcal G_{\theta^\ast}=I-DT(\theta^\ast),
\]
which coincides with both the missing-information ratio and the
information-geometric Hessian of the observed likelihood.

From this operator we derive two acceleration strategies.
The \textbf{G-Accelerator} uses the spectral gap to obtain an optimal
Nesterov-type momentum $\beta^* = (1-\sqrt{\lambda_*})/(1+\sqrt{\lambda_*})$.
The \textbf{Geo-Adaptive} accelerator extends the geometric EM framework
of Zhou, Alexander \& Lange by replacing their fixed correction strength
$\gamma=8$ with the adaptive rule $\gamma_k = 1/\hat\lambda_k$, where
$\hat\lambda_k$ is estimated online from the parameter trajectory.
Both methods are parameter-free; Geo-Adaptive achieves dramatic
acceleration precisely when the spectral gap is smallest.

Numerical experiments on Gaussian mixtures demonstrate that both
accelerators consistently outperform standard EM and fixed-$\gamma$ DCC-EM,
with Geo-Adaptive attaining speedups exceeding $8\times$ in the most
challenging regimes.
\end{abstract}

\section{Introduction}
\label{sec:intro}

The expectation--maximization (EM) algorithm remains one of the
fundamental methods for likelihood optimization in latent-variable
models \cite{NealHinton1998}.
Its appeal lies in the fact that each iteration is simple,
monotone, and often computationally tractable even when direct
maximization of the observed-data likelihood is difficult.

Despite its classical status, the mathematical understanding of EM
presents a structural tension.
Global monotonicity is inherently nonlinear, whereas local
contraction is infinitesimal, and acceleration methods are
typically designed from empirical contraction behavior.
What has remained less transparent is how these three levels of
description fit into a single dynamical picture, and which object
serves as the bridge between them.

The latent-variable structure of EM provides this bridge.
Unlike a generic fixed-point iteration, an EM step carries an
intrinsic posterior transport mechanism that couples the
likelihood increment to the evolution of posterior distributions.
Making this coupling explicit reveals the operator that governs
both local contraction and information-geometric curvature
simultaneously.

The purpose of this paper is to identify and exploit this dynamical
structure.
Our starting point is a global energy decomposition in which the
likelihood increment is expressed by two posterior-relative entropy
terms.
Although algebraically related to the classical decomposition of
Dempster, Laird, and Rubin, its dynamical role is different: along
the EM trajectory, the posterior KL term becomes a transport observable
that bridges nonlinear descent and local rigidity.

Linearizing around a nondegenerate maximizer $\thetastar$,
we identify the local operator
\[
\cG_{\thetastar}=I-DT(\thetastar).
\]
Three structures coincide exactly at this operator: the linearized EM
dynamics, the missing-information ratio, and the information-geometric
Hessian of the observed likelihood.
This is the exact meeting point of dynamics, information, and geometry.

From this operator, two acceleration strategies follow naturally.
The \textbf{G-Accelerator} uses the spectral gap to derive an optimal
Nesterov-type momentum.
The \textbf{Geo-Adaptive} accelerator extends the geometric EM framework
of Zhou, Alexander \& Lange \cite{ZhouAlexanderLange2011} (DCC-EM) by
replacing their fixed correction strength $\gamma=8$ with the adaptive
rule $\gamma_k = 1/\hat\lambda_k$, where $\hat\lambda_k$ is estimated
online from the parameter trajectory.

The main contribution is a unified dynamical framework for EM connecting
global descent, local spectral behavior, geometric curvature, and two
principled acceleration methods through a single latent-variable operator.

\paragraph{Relation to prior work.}
Classical analyses of EM established monotonicity of the observed
likelihood and characterized local convergence through the
missing-information principle
\cite{Dempster1977,Wu1983,Louis1982,MengRubin1993}.
Amari~\cite{Amari2016} embedded EM in information geometry.
A detailed performance comparison with SQUAREM \cite{VaradhanRoland2008}
is beyond the scope of this paper.
A closely related $\mathcal{G}$-theory has been developed for the
Blahut--Arimoto algorithm \cite{wang2026exact}.
Zhou, Alexander and Lange \cite{ZhouAlexanderLange2011} proposed
a quasi-Newton acceleration for EM-type algorithms known as DCC-EM
(Differential geometric Conjugate gradient EM).  Their method uses
a fixed correction strength $\gamma=8$.  The present work identifies
the spectral gap as the controlling quantity and introduces
\textbf{Geo-Adaptive}, a spectrally adaptive variant that dynamically
adjusts $\gamma_k = 1/\hat\lambda_k$ to achieve optimal acceleration.
Ramakrishnan et al.~\cite{Ramakrishnan2021} applied EM-type iterations
to compute quantum channel capacities, highlighting the importance
of efficient acceleration in high-dimensional settings.

A separate line of research by Hayashi \cite{Hayashi2023,Hayashi2024,Hayashi2025}
has developed information-geometric acceleration methods for EM and related
algorithms using Bregman divergences.  A detailed comparison with these
methods is beyond the scope of this paper, as they involve sophisticated
geometric structures and problem-specific parameter tuning that require
a dedicated numerical study.  Such a comparison will be presented elsewhere.

\paragraph{Organization of the paper.}
Section~\ref{sec:setup} sets up the EM dynamical system and the
relaxation operator.
Section~\ref{sec:triple} proves the triple equivalence.
Section~\ref{sec:energy} develops the global energy law and its
local quadratic form.
Section~\ref{sec:rigidity} analyses rigidity and the
Blahut--Arimoto connection.
Section~\ref{sec:acceleration} presents the G-Accelerator:
optimal momentum from the spectral gap.
Section~\ref{sec:geo-accel} presents the Geo-Adaptive accelerator:
spectrally adaptive geometric EM.
Section~\ref{sec:convergence} provides the convergence theory.
Section~\ref{sec:experiments} presents numerical experiments on
Gaussian mixtures, and Section~\ref{sec:discussion} concludes with
open problems.

\section{Setup and Main Results}
\label{sec:setup}
\subsection{Framework and the relaxation operator}
\label{sec:framework}

Let \(x\) be observed data and \(z\) a latent (unobserved) variable taking
values in some space.  The parametric model is specified by the joint
density \(p(x,z\mid\theta)\) with \(\theta\in\Theta\subset\bbR^{d}\).  
The \emph{observed-data log-likelihood} and the \emph{auxiliary function}
of the EM algorithm are, respectively,
\[
  \ell(\theta)=\log p(x\mid\theta),\qquad
  Q(\theta'\mid\theta)=\mathbb{E}_{p(z\mid x,\theta)}
  \bigl[\log p(x,z\mid\theta')\bigr].
\]
One step of the EM algorithm consists of evaluating
\(Q(\cdot\mid\theta)\) (the E‑step) and then setting
\[
  T(\theta)=\arg\max_{\theta'} Q(\theta'\mid\theta)
\]
(the M‑step).  The map \(T:\Theta\to\Theta\) completely describes the
deterministic dynamics \(\theta_{k+1}=T(\theta_k)\) of the algorithm.

\begin{assumption}[Regular fixed point]\label{ass:regular}
\(\theta^{*}\) is a fixed point of \(T\) (i.e.\ \(T(\theta^{*})=\theta^{*}\)),
the map \(T\) is twice continuously differentiable in a neighbourhood of
\(\theta^{*}\), and the complete‑data Fisher information
\(\Icom(\theta^{*})\) (defined below) is positive definite.
The complete‑data model \(p(x,z\mid\theta)\) satisfies the standard regularity
conditions that allow differentiation under the integral sign.
\end{assumption}

\begin{assumption}[Regular complete-data model]\label{ass:complete}
The complete-data density \(p(x,z\mid\theta)\) is twice continuously
differentiable in \(\theta\) for almost every \((x,z)\), and its first and
second derivatives admit integrable envelopes that allow differentiation
under the integral sign.  Moreover, for every \(\theta\) the posterior
distribution \(p(z\mid x,\theta)\) is absolutely continuous with respect
to a common dominating measure, and the mapping
\(\theta\mapsto p(z\mid x,\theta)\) is continuous in total variation.
\end{assumption}

Three Fisher information matrices are central to the analysis.  The
\emph{complete‑data Fisher information} is the expected negative Hessian of
the complete‑data log‑likelihood under the posterior distribution of the
latent variable at the fixed point:
\[
  \Icom(\theta^{*})
  :=-\mathbb{E}_{p(z\mid x,\theta^{*})}\!\bigl[\nabla_{\theta}^{2}
       \log p(x,z\mid\theta^{*})\bigr].
\]
The \emph{observed Fisher information} is the negative Hessian of the
observed log‑likelihood,
\[
  \Iobs(\theta^{*})
  :=-\nabla_{\theta}^{2}\,\ell(\theta^{*}),
\]
and the \emph{missing information} is their difference,
\[
  \Imis(\theta^{*})
  :=\Icom(\theta^{*})-\Iobs(\theta^{*}).
\]
All three matrices are symmetric and, under Assumption~\ref{ass:regular},
positive semi‑definite with \(\Icom(\theta^{*})\) strictly
positive definite.  We equip the parameter space with the Riemannian metric
induced by the complete‑data information,
\[
  \langle u,v\rangle_{\Icom}
  := u^{\top}\Icom(\theta^{*})\,v .
\]

Because the EM algorithm has been cast as the discrete dynamical system
\(\theta_{k+1}=T(\theta_k)\), its local convergence behaviour near a fixed
point is governed by the \textbf{Jacobian matrix} of \(T\).  We write
\[
  DT(\theta^{*}) \;:=\; \frac{\partial T(\theta)}{\partial\theta}
                        \bigg|_{\theta=\theta^{*}}
  \;\in\; \mathbb{R}^{d\times d}.
\]
Applying the implicit function theorem to the first‑order condition of the
M‑step yields the classic \emph{missing‑information principle} of
Dempster, Laird and Rubin~\cite{Dempster1977}:
\begin{equation}\label{eq:DT_mis}
  DT(\theta^{*}) \;=\; \Icom(\theta^{*})^{-1}\,
                        \Imis(\theta^{*}).
\end{equation}
In words, the linearisation of the EM map is exactly the ratio of the
missing information to the complete information.

The triple equivalence at the heart of this paper originates from the
\textbf{relaxation operator}, which measures the fraction of the error
that EM still leaves uncorrected after a single step.

\begin{definition}[Relaxation operator]\label{def:G}
The \emph{relaxation operator} at a regular fixed point is
\[
  \cG_{\theta^{*}} \;:=\; I - DT(\theta^{*}).
\]
\end{definition}

Using~\eqref{eq:DT_mis} we immediately obtain the second of the three
equivalent representations:
\[
  \cG_{\theta^{*}}
  = I - \Icom(\theta^{*})^{-1}\Imis(\theta^{*})
  = \Icom(\theta^{*})^{-1}
    \bigl(\Icom(\theta^{*})-\Imis(\theta^{*})\bigr)
  = \Icom(\theta^{*})^{-1}\Iobs(\theta^{*}).
\]
Thus \(\cG_{\theta^{*}}\) is simultaneously the linearised
contraction defect, the observed‑to‑complete information ratio, and, as
will be shown in Theorem~\ref{thm:triple}, the Riemannian Hessian of the
observed log‑likelihood under the complete‑data metric.  This single
operator will be the organising centre of all subsequent developments.

\begin{remark}
The similar operator-level identification has been noted in the
Blahut--Arimoto algorithm \cite{wang2026exact}.
\end{remark}

\subsection{Main theorems}

The following five theorems are the main results of this paper.
A unified convergence framework is presented in
Section~\ref{sec:convergence}.  Proofs of Theorems~\ref{thm:triple}--%
\ref{thm:acceleration} are in Appendix~\ref{app:proofs}.

\begin{theorem}[Triple equivalence]\label{thm:triple}
Under Assumption~\ref{ass:regular},
\begin{equation}\label{eq:triple}
  \cG_{\thetastar}
  \;=\;
  I - DT(\thetastar)
  \;=\;
  \Icom(\thetastar)^{-1}\Iobs(\thetastar)
  \;=\;
  \mathrm{Hess}_{\Icom}\,\ell(\thetastar).
\end{equation}
\end{theorem}

A direct decomposition of the log-likelihood difference along the EM trajectory
yields the following exact identity.  It separates the one-step gain into the
M-step ascent and the posterior transport cost---two structurally independent
dynamical quantities that will become the foundation of the dissipation and
acceleration analysis.

\begin{theorem}[Global energy law: a dynamical decomposition along the EM trajectory]\label{thm:global} Under Assumptions~\ref{ass:regular} and~\ref{ass:complete}, for every $\theta\in\Theta$,
\begin{equation}\label{eq:global_energy}
  \ell(T(\theta)) - \ell(\theta)
  \;=\;
  \bigl[Q(T(\theta)\mid\theta) - Q(\theta\mid\theta)\bigr]
  +
  \KL\!\bigl(p(z\mid x,\theta)\;\|\;p(z\mid x,T(\theta))\bigr).
\end{equation}
In particular, $\ell(T(\theta))\ge\ell(\theta)$, with equality if and only
if $T(\theta)=\theta$.  Under the additional hypothesis that $\theta^*$ is
the unique fixed point of $T$ in the connected component containing the
trajectory, the equality condition is equivalent to $\theta=\theta^*$.
Near $\thetastar$, setting $u=\theta-\thetastar$,
\begin{equation}\label{eqn:k-1}
  \ell(T(\theta)) - \ell(\theta)
  \;=\;
  \tfrac12 \, u^\top \Icom(\thetastar)\bigl(2\cG_{\thetastar}^2 - \cG_{\thetastar}^3\bigr) u
  + O(\|u\|^3).
\end{equation}
Equivalently, since $\cG_{\thetastar} = I - DT(\thetastar)$, this can be
written as $\tfrac12 u^\top \Icom(\thetastar)\,(I-DT(\thetastar))(I-DT(\thetastar)^2)\,u$,
showing that the one-step energy gain is controlled by the product of the
single-step and two-step relaxation defects.
\end{theorem}

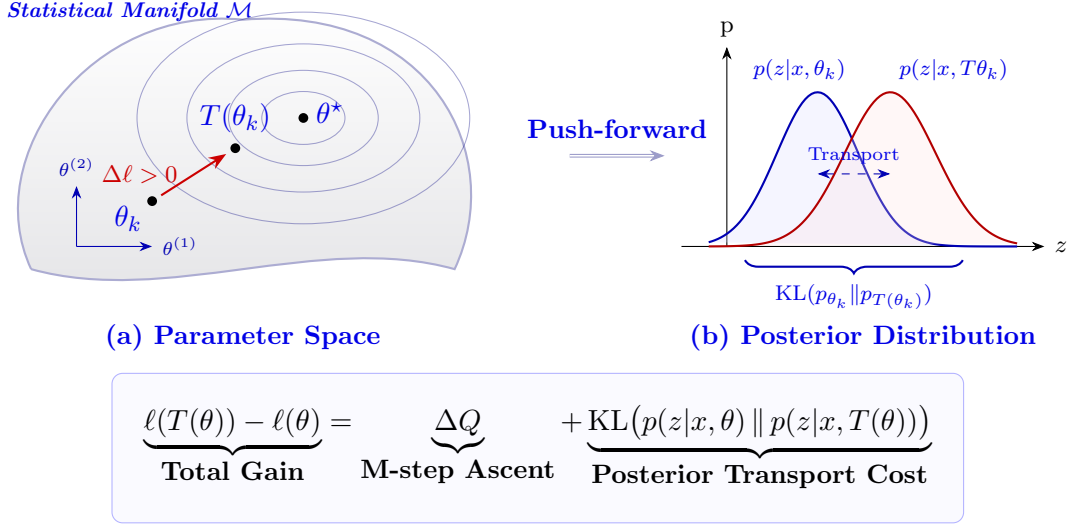
\begin{figure}[t]
\centering
\begin{tikzpicture}[
    scale=1.0,
    >=Stealth,
    param/.style={circle, fill=black, inner sep=1.3pt},
    it_arrow/.style={->, thick, color=red!80!black, shorten >=2pt, shorten <=2pt},
    manifold/.style={top color=white, bottom color=blue!5!gray!20, draw=blue!30!gray!50, thick},
    contour/.style={draw=blue!50!black, opacity=0.35, thin},
    density_k/.style={thick, color=blue!70!black},
    density_next/.style={thick, color=red!70!black}
]

\begin{scope}[shift={(0,0)}]
    \draw[manifold] (-2.8,-1.5) .. controls (-1, -2) and (1, -1) .. (2.8, -1.5)
                    .. controls (3.5, 0) and (2.5, 2) .. (0, 1.8)
                    .. controls (-2, 2) and (-3.5, 0) .. (-2.8, -1.5);
    
    \node[color=blue!90!black, font=\scriptsize\bfseries\itshape] at (-1.5, 1.9) {Statistical Manifold $\mathcal{M}$};

    \foreach \r in {0.5, 0.9, 1.4, 2.0}
        \draw[contour] (0.8, 0.5) ellipse ({\r*1.1} and {\r*0.7});
    
    \node[param, label={[right, xshift=1pt, color=blue!90!black]$\theta^\star$}] (star) at (0.8, 0.5) {};
    \node[param, label={[below left, color=blue!90!black]$\theta_k$}] (tk) at (-1.2, -0.6) {};
    \node[param, label={[above, color=blue!90!black] $T(\theta_k)$}] (Ttk) at (-0.1, 0.1) {};
    
    \draw[it_arrow] (tk) -- (Ttk) node[midway, left, font=\footnotesize, xshift=-2pt] {$\Delta \ell > 0$};

    \draw[->, blue!70!black, thin] (-2.2,-1.2) -- (-1.2,-1.2) node[right, font=\tiny] {$\theta^{(1)}$};
    \draw[->, blue!70!black, thin] (-2.2,-1.2) -- (-2.2,-0.4) node[above, font=\tiny] {$\theta^{(2)}$};
    
    \node[font=\small\bfseries, color=blue!90!black] (label-a) at (0, -2.4) {(a) Parameter Space};
\end{scope}

\draw[->, double, color=blue!30!gray!60, shorten >=12pt, shorten <=12pt] (3.9, 0) -- (6.0, 0) 
    node[midway, above, font=\small\bfseries, color=blue!90!black, yshift=3pt] {Push-forward};

\begin{scope}[shift={(6.4, -1.2)}, scale=1.2] 
    \draw[->, thin] (-0.5,0) -- (3.5,0) node[right, font=\footnotesize] {$z$};
    \draw[->, thin] (0,0) -- (0,2.2) node[above, font=\footnotesize] {p};

    \draw[density_k, fill=blue!12, fill opacity=0.3] 
        plot[domain=-0.2:3.2, samples=100] (\x, {1.7*exp(-(\x-1.0)^2/0.4)}) ;
    \node[color=blue!90!black, font=\scriptsize\bfseries] at (0.8, 1.95) {$p(z|x,\theta_k)$};

    \draw[density_next, fill=red!12, fill opacity=0.25] 
        plot[domain=-0.2:3.2, samples=100] (\x, {1.7*exp(-(\x-1.8)^2/0.5)});
    \node[color=blue!90!black, font=\scriptsize\bfseries] at (2.5, 1.95) {$p(z|x,T\theta_k)$};

    \draw[<->, dashed, color=blue!70!black] (1.0, 0.8) -- (1.8, 0.8) 
        node[midway, above, font=\tiny, color=blue!90!black] {Transport};

    \draw[decorate, decoration={brace, amplitude=5pt, mirror}, thick, blue!70!black] 
        (0.2, -0.15) -- (2.6, -0.15) 
        node[midway, below, yshift=-5pt, font=\scriptsize\bfseries, color=blue!90!black] {$\mathrm{KL}(p_{\theta_k} \| p_{T(\theta_k)})$};
\end{scope}

\node[font=\small\bfseries, color=blue!90!black] at (8.2, -2.4) {(b) Posterior Distribution};

\node[draw=blue!30, fill=blue!2, rounded corners, inner sep=12pt, below=2.2cm of tk, xshift=5.1cm] (eq) {
    $\displaystyle \underbrace{\ell(T(\theta)) - \ell(\theta)}_{\text{\small\bfseries Total Gain}} = 
    \underbrace{\Delta Q}_{\text{\small\bfseries M-step Ascent}} + 
    \underbrace{\mathrm{KL}\bigl(p(z|x,\theta) \,\|\, p(z|x,T(\theta))\bigr)}_{\text{\small\bfseries Posterior Transport Cost}}$
};

\end{tikzpicture}
\caption{The Information-Geometric View of EM Dynamics. The relaxation operator $T$ maps parameter updates on (a) the manifold to posterior transports in (b) the functional space, satisfying the global energy law.}
\label{fig:energy_geometry_final_v5}
\end{figure}

\begin{remark}[Relation to the classical DLR identity]
Algebraically, \eqref{eq:global_energy} coincides with the DLR expansion
evaluated at $\theta'=T(\theta)$.
The novelty of Theorem~\ref{thm:global} is structural, not algebraic.
In the DLR argument \cite{Dempster1977}, the second KL argument is a
generic $\theta'$, and the non-negativity of the KL term is used only
to establish $\ell(\theta') \ge \ell(\theta)$.
In contrast, Theorem~\ref{thm:global} fixes the second argument to be
$T(\theta)$ and reads the identity \emph{along the trajectory}: the KL
term becomes an exact, pointwise measure of posterior displacement after
one step, separating the energy change into two structurally independent
components---M-step ascent and posterior transport cost.
This separation is the foundation of Theorem~\ref{thm:rigidity} and
of the link between energy gain and the spectral properties of $\cG$.
\end{remark}

\begin{theorem}[Rigidity]\label{thm:rigidity}
Under Assumptions~\ref{ass:regular} and~\ref{ass:complete}, the following are equivalent:
\begin{enumerate}[label=(\roman*)]
\item $\KL\!\bigl(p(z\mid x,\theta)\;\|\;p(z\mid x,T(\theta))\bigr) = 0$;
\item $p(z\mid x,T(\theta)) = p(z\mid x,\theta)$ \emph{(posterior invariance)};
\item $\bbE_{p(z\mid x,T(\theta))}[T(x,z)]
      = \bbE_{p(z\mid x,\theta)}[T(x,z)]$
      \emph{(sufficient-statistic locking)};
\item the EM velocity is directionally constant along the trajectory,
      i.e.\ $\frac{T(\theta)-\theta}{\|T(\theta)-\theta\|}$ remains invariant
      \emph{(velocity locking)}.
\end{enumerate}
For exponential-family complete-data models, these conditions are further
equivalent to the EM step being compatible with Amari's dual-affine structure
\emph{(dual-flat regime)}.  In the Blahut--Arimoto setting, condition~(i)
holds globally and identically \cite{wang2026exact}.
\end{theorem}

\begin{proposition}[Spectral boundedness of the relaxation operator]\label{prop:spectral_bounds}
Under Assumption~\ref{ass:regular}, let $\cG_{\thetastar} = \Icom(\thetastar)^{-1}\Iobs(\thetastar)$.
Then all eigenvalues $\lambda_i(\cG_{\thetastar})$ satisfy
\begin{equation}\label{eq:spectral_bounds}
  0 \;\le\; \lambda_i(\cG_{\thetastar}) \;\le\; 1 .
\end{equation}
Equivalently, the spectral radius of the EM Jacobian satisfies $\rho(DT(\thetastar)) \le 1$.
\end{proposition}
\begin{proof}
The missing-information principle~\cite{Dempster1977} and the Louis (1982) identity~\cite{Louis1982} give
\[
  \Imis(\thetastar)
  = \text{Cov}_{p(z\mid x,\thetastar)}\!\bigl[\nabla_\theta \log p(x,z\mid\theta^*)\bigr],
\]
which is a covariance matrix and therefore positive semi‑definite:
$\Imis(\thetastar) \succeq 0$ (Loewner order).
Consequently,
\[
  \Iobs(\thetastar) = \Icom(\thetastar) - \Imis(\thetastar) \;\preceq\; \Icom(\thetastar).
\]
Since $\Icom(\thetastar)$ is positive definite, we may consider its invertible square root
$\Icom^{1/2}$, and the above inequality yields
\[
  0 \;\preceq\; \Icom^{-1/2}\,\Iobs\,\Icom^{-1/2} \;\preceq\; I .
\]
The matrix $\Icom^{-1/2}\,\Iobs\,\Icom^{-1/2}$ is symmetric and therefore has real eigenvalues
all lying in $[0,1]$.  Now,
\[
  \cG_{\thetastar} = \Icom(\thetastar)^{-1}\Iobs(\thetastar)
  = \Icom^{-1/2}\bigl(\Icom^{-1/2}\,\Iobs\,\Icom^{-1/2}\bigr)\Icom^{1/2}
\]
is similar to $\Icom^{-1/2}\,\Iobs\,\Icom^{-1/2}$, hence shares the same
spectrum.  Thus $0 \le \lambda_i(\cG_{\thetastar}) \le 1$ for all $i$.
Finally, $\rho(DT(\thetastar)) = \rho(I - \cG_{\thetastar}) = \max_i |1 - \lambda_i| \le 1$.
\end{proof}

\begin{remark}[Statistical meaning]
  Proposition~\ref{prop:spectral_bounds} is a rigorous
  consequence of the fact that the missing information equals the conditional covariance of
  the complete-data score.  It formalises the intuition that, at a maximiser of the observed
  likelihood, the observed information $I_{\mathrm{obs}}$ cannot exceed the complete
  information $I_{\mathrm{com}}$.  The bounds $0 \le \lambda_i \le 1$ are the spectral
  counterpart of this classical inequality.
\end{remark}

\begin{theorem}[Spectral gap and EM convergence rate]\label{thm:spectral}
Under Assumptions~\ref{ass:regular} and~\ref{ass:complete}, and with
Proposition~\ref{prop:spectral_bounds} ensuring that
the eigenvalues of $\cG_{\thetastar}$ lie in $[0,1]$, the spectral radius of $DT(\thetastar)$
in the $\Icom$-norm is
\begin{equation}\label{eq:spectral_radius}
  \rho_{\mathrm{EM}} = 1 - \lmin(\cG_{\thetastar})
  = 1 - \lmin\!\bigl(\Icom(\thetastar)^{-1}\Iobs(\thetastar)\bigr).
\end{equation}
The slowest mode corresponds to the direction of least observable information.
\end{theorem}

\section{The Relaxation Operator: Triple Equivalence}
\label{sec:triple}

Theorem~\ref{thm:triple} asserts three representations of the same
operator.  We unpack each in turn and explain why the third is the
non-obvious one.

\paragraph{First representation: linearised contraction defect.}
The Jacobian $DT(\thetastar)$ governs the local convergence rate.  Its
complement $\cG_{\thetastar} = I - DT(\thetastar)$ measures how much each
eigenmode is ``corrected'' per EM step.  This is purely a dynamical
statement.

\paragraph{Second representation: information ratio.}
The classical missing-information principle gives $DT(\thetastar) =
\Icom^{-1}\Imis$, hence
\[
  \cG_{\thetastar}
  = I - \Icom^{-1}\Imis
  = \Icom^{-1}(\Icom - \Imis)
  = \Icom^{-1}\Iobs.
\]
This identification is essentially contained in~\cite{Dempster1977} but its
operator-theoretic role as a geometric object has not been made explicit.

\paragraph{Third representation: Riemannian Hessian.}
This is the new content of Theorem~\ref{thm:triple}.  The key steps are:
(a)~at the critical point $\thetastar$, the Riemannian and Euclidean
Hessians of $\ell$ differ only by Christoffel terms that vanish at critical
points; (b)~the Euclidean Hessian of $\ell$ is $-\Iobs(\thetastar)$;
(c)~in the $\Icom$-metric, this becomes $\Icom^{-1}\Iobs$.  Thus
$\cG_{\thetastar} = \mathrm{Hess}_{\Icom}\,\ell(\thetastar)$.

\medskip

The significance of this third representation is that it connects the local
\emph{dynamics} of EM (governed by $DT$) to the local \emph{geometry} of
the likelihood landscape (governed by the Hessian).  In a gradient-flow
setting these would be identical by definition; for EM they coincide only
because $\Icom$ is precisely the right metric.

The proof is in Appendix~\ref{app:triple}.

\begin{remark}[Nature of the novelty]
Each of the three representations in Theorem~\ref{thm:triple} is individually
known: the Jacobian $DT(\theta^*)$ from DLR (1977), the information ratio
$\Icom^{-1}\Iobs$ from the missing-information principle, and the Riemannian
Hessian from information geometry.  What is new is their \emph{exact
operator-level coincidence} at a regular fixed point, which implies that
the EM contraction defect, the observed-to-complete information ratio, and
the observed-likelihood curvature are one and the same symmetric operator
on the parameter space.  This unification, rather than any single identity,
is the foundation of the present theory.
\end{remark}

\begin{lemma}[Self-adjointness of the relaxation operator]\label{lem:selfadjoint}
Under Assumption~\ref{ass:regular}, the relaxation operator
\(\cG_{\theta^*}\) is self-adjoint with respect to the
\(\Icom\)-inner product, i.e.\
\(\Icom\,\cG_{\theta^*} = \cG_{\theta^*}^{\top}\Icom\).
\end{lemma}
\begin{proof}
From Theorem~\ref{thm:triple},
\(\cG_{\theta^*} = \Icom(\theta^*)^{-1}\Iobs(\theta^*)\).
Both \(\Icom\) and \(\Iobs\) are symmetric matrices, hence
\[
\Icom\,\cG_{\theta^*}
= \Icom\,\Icom^{-1}\Iobs = \Iobs
= \Iobs^{\top}
= (\Icom^{-1}\Iobs)^{\top}\Icom
= \cG_{\theta^*}^{\top}\Icom.
\]
\end{proof}

\section{Global Energy Law and Local Quadratic Structure}
\label{sec:energy}
The global energy decomposition derived in this section serves as the
bridge from nonlinear monotonicity to the local operator that governs
the infinitesimal EM dynamics. The theorem~\ref{thm:global} makes two separate statements.

\paragraph{Global exact decomposition.}
The identity
\[
  \underbrace{\ell(T(\theta))-\ell(\theta)}_{\text{one-step gain}}
  =
  \underbrace{\Delta Q}_{\text{M-step ascent}}
  +
  \underbrace{\KL(\pi_\theta\|\pi_{T(\theta)})}_{\text{posterior transport cost}}
\]
is exact and holds at every iterate, not just near $\thetastar$.  It reveals
that EM monotonicity has two distinct mechanisms: direct maximisation of the
surrogate $Q$, and posterior-alignment cost.

\begin{remark}[Relation to existing decompositions]
  The identity
  \[
    \ell(T(\theta))-\ell(\theta)
    = \Delta Q + \KL(p(z\mid x,\theta)\,\|\,p(z\mid x,T(\theta)))
  \]
  can be obtained algebraically from the standard DLR expansion (see
  Appendix~\ref{app:energy}).  The novelty lies not in the algebraic form
  as such, but in (i)~the explicit interpretation of the KL term as a
  \emph{dynamical posterior coupling loss} (its second argument is the
  next-step posterior, not the current one), and (ii)~the local quadratic
  expansion
  \[
    \ell(T(\theta))-\ell(\theta)
    = \tfrac12 u^\top \Icom(\thetastar)\bigl(2\cG_{\thetastar}^2 - \cG_{\thetastar}^3\bigr) u + O(\|u\|^3),
  \]
  which directly links the one-step energy gain to the relaxation operator
  that governs the linearised convergence rate.
  This connection is absent from existing EM decompositions, including the
  free-energy formulation of Neal and Hinton~\cite{NealHinton1998} and the
  standard ELBO identity.
\end{remark}

\paragraph{Comparison with Blahut--Arimoto.}
In the BA algorithm, the complete-data model has Gibbs structure, which
forces the posterior coupling term to vanish identically.  This is why BA
admits a globally exact energy law of the form $\Delta\ell = -\chi^2$,
whereas general EM only recovers such a structure locally.  Rigidity
(Theorem~\ref{thm:rigidity}) characterises exactly when the EM posterior
coupling term likewise vanishes.  A complete $\mathcal{G}$-theory for
BA dynamics is developed in \cite{wang2026exact}.

\paragraph{Local quadratic structure.}
Let $u = \theta-\thetastar$ and $\delta = T(\theta)-\thetastar$.  Expanding
$\ell$ directly around $\thetastar$ yields
\[
  \ell(\theta) = \ell(\thetastar) - \tfrac12 u^\top \Iobs u + O(\|u\|^3),
  \qquad
  \ell(T(\theta)) = \ell(\thetastar) - \tfrac12 \delta^\top \Iobs \delta + O(\|\delta\|^3).
\]
Using $\delta = DT u + O(\|u\|^2)$ and $DT = I - \cG$,
\[
  \ell(T(\theta))-\ell(\theta)
  = \tfrac12 u^\top\bigl(\Iobs - DT^\top \Iobs DT\bigr) u + O(\|u\|^3).
\]
With $\Iobs = \Icom \cG$ and the $\Icom$-self-adjointness of $\cG$,
$DT^\top \Iobs DT = \Icom(\cG - 2\cG^2 + \cG^3)$, hence we get \eqref{eqn:k-1}
\[
  \ell(T(\theta))-\ell(\theta)
  = \tfrac12 u^\top \Icom\bigl(2\cG_{\thetastar}^2 - \cG_{\thetastar}^3\bigr) u + O(\|u\|^3).
\]
The term $2\cG^2 - \cG^3 = (I-DT)(I-DT^2)$ reveals that the energy gain
involves both the immediate contraction defect $I-DT$ and the compounded
effect $I-DT^2$.  (Note that $I-DT^2 = (I-DT)(I+DT)$, so the energy gain
factorises into the immediate relaxation defect and a compounded two-step
correction.)  For a slow eigen-direction with eigenvalue $\lambda \ll 1$,
the energy gain scales as $\sim \lambda^2 \|u\|^2$ rather than $\sim \lambda \|u\|^2$,
which further reinforces the spectral bottleneck.

\section{Spectral and Rigidity Consequences of the Relaxation Operator}
\label{sec:rigidity}
The spectral structure of the relaxation operator quantifies both the contraction rate and the degree of posterior rigidity along the EM trajectory.

\subsection{The rigidity equivalence}

Theorem~\ref{thm:rigidity} characterises the regime of ``perfect
dissipation''---when the posterior coupling term vanishes.

The equivalence chain (i)$\Leftrightarrow$(ii)$\Leftrightarrow$(iii) is
algebraic: the KL term vanishes iff the two posteriors are identical iff
the sufficient-statistic expectations do not change under one EM step
(because the M-step is determined solely by these expectations).

The equivalence with (iv) (velocity locking) follows from the fact that
if the posterior is unchanged, then the M-step target does not move along
the interpolation path, forcing the velocity direction to be constant.

\subsection{Amari's dual-flat geometry as a dynamical condition}

For exponential-family models, the rigidity conditions are further
equivalent to the EM step remaining inside an $e$-flat submanifold in the
sense of Amari.  This gives:
\begin{quote}
Amari's static dual-flatness condition $=$ dynamical posterior invariance
$=$ exact one-step dissipation.
\end{quote}
This is the dynamical version of Amari's projection theorem: rather than
describing each EM step as a projection, it identifies the conditions under
which the \emph{entire trajectory} obeys an exact dissipative law.

\subsection{Blahut--Arimoto as the globally rigid limit}

In the BA algorithm, the complete-data model has the Gibbs form
$p(x,z\mid\theta) = p(z)\exp(\theta^\top T(x,z) - A(\theta))$.
A direct calculation shows that the posterior coupling term then vanishes
for \emph{every} $\theta$, not just at the fixed point.
Consequently, the BA dynamics admits a globally exact $\chi^2$
dissipation identity \cite{wang2026exact}. General EM recovers
this only asymptotically near $\thetastar$.

The rigidity theorem thus places BA and EM in a common framework and
identifies their precise structural difference: BA is the globally rigid
limit, EM is the general dissipative regime.  The posterior-coupling KL
term in Theorem~\ref{thm:global} is the exact measure of the deviation from
this rigid limit.  When it vanishes identically---a condition fully
characterised by Theorem~\ref{thm:rigidity}---the EM dynamics collapses to
the dual-projection structure of Csisz\'ar and Tusn\'ady~\cite{Csiszar1984}
and the exact dissipation law of the BA $\mathcal{G}$-theory.  When it does
not, the trajectory departs from the projection orbit, and the KL term
quantifies the energetic cost of this departure.

The proof is in Appendix~\ref{app:rigidity}.

\section{Optimal Momentum Acceleration: The G-Accelerator}
\label{sec:acceleration}

The spectral gap $\lmin$ directly determines the asymptotic contraction
rate of EM.  This section shows that it also yields an optimal momentum
rule for a Nesterov-type accelerator.

\subsection{Optimal momentum from the spectral gap}

Consider the momentum iteration
\[
\tilde\theta_k = \theta_k + \beta(\theta_k - \theta_{k-1}),\qquad
\theta_{k+1} = T(\tilde\theta_k).
\]
Linearizing $T$ around $\thetastar$, a single mode with eigenvalue
$\mu = 1-\lambda$ obeys the scalar recurrence
\[
u_{k+1} = \mu(1+\beta)u_k - \mu\beta u_{k-1},
\]
with characteristic polynomial $r^2 - \mu(1+\beta)r + \mu\beta = 0$.
Minimizing the spectral radius over $\beta$ for the worst-case
$\lambda \in [\lmin,\lmax]$ yields the Chebyshev-optimal parameter
\begin{equation}\label{eq:optimal_momentum}
\beta^* = \frac{1-\sqrt{\lmin}}{1+\sqrt{\lmin}},\qquad
\rho_{\mathrm{acc}} = 1-\sqrt{\lmin}.
\end{equation}
The derivation follows the standard Chebyshev acceleration argument
and is detailed in Appendix~\ref{app:momentum}.  The improvement over
standard EM is $\rho_{\mathrm{EM}}/\rho_{\mathrm{acc}} = (1-\lmin)/(1-\sqrt{\lmin})$,
which diverges as $\lmin \to 0$.

\begin{figure}[t]
    \centering
    \includegraphics[width=0.85\textwidth]{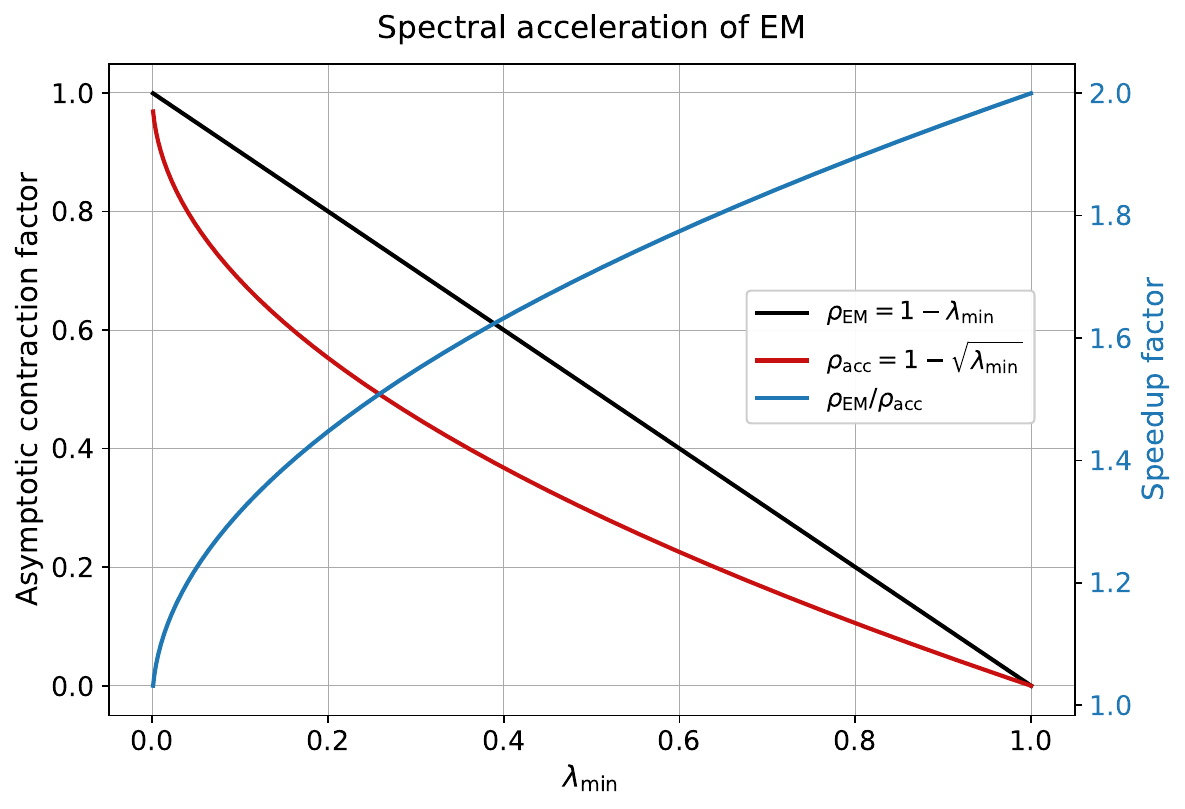}
    \caption{Theoretical contraction factors and speedup as functions of the
             spectral gap $\lmin=\lambda_{\min}(\cG_{\theta^*})$.
             Solid black: EM rate $\rho_{\mathrm{EM}}=1-\lmin$;
             solid red: accelerated rate $\rho_{\mathrm{acc}}=1-\sqrt{\lmin}$;
             solid blue (right axis): speedup factor $\rho_{\mathrm{EM}}/\rho_{\mathrm{acc}}$.
             As $\lmin\to 0$, the acceleration factor diverges.}
    \label{fig:spectral_acceleration}
\end{figure}

Figure~\ref{fig:spectral_acceleration} visualises these theoretical predictions.

\subsection{Practical spectral estimation}

To implement the G-Accelerator, only the smallest eigenvalue $\lmin$
of $\cG_{\theta^*}$ is required.  In practice, we estimate it using a
Krylov-based spectral estimation procedure applied to the local operator
$\cG_{\theta_k}$.  This requires a constant number of additional EM
evaluations per iteration.  A detailed description of this general
spectral estimation framework for Gibbs-type dynamics can be found in
\cite{wang2026exact}.

\section{The Geo-Adaptive Accelerator: Spectral Adaptation of Geometric EM}
\label{sec:geo-accel}

\subsection{The DCC-EM Framework of Zhou, Alexander \& Lange (Fixed $\gamma$)}

Zhou, Alexander and Lange \cite{ZhouAlexanderLange2011} proposed a 
quasi-Newton acceleration for EM-type algorithms, which they called 
DCC-EM (Differential geometric Conjugate gradient EM).  Their iteration 
takes the form

\[
\theta_{k+1} = M(\theta_k) + \gamma \langle r_k, u_k \rangle u_k,
\]

where:
\begin{itemize}
    \item $M(\theta_k)$ is the standard EM step (M-step);
    \item $r_k = M(\theta_k) - \theta_k$ is the EM residual;
    \item $u_k = \dfrac{\theta_k - \theta_{k-1}}{\|\theta_k - \theta_{k-1}\|}$ 
          is the unit tangent vector along the recent parameter trajectory;
    \item $\gamma > 0$ is a \textbf{fixed} parameter (Zhou et al. recommended 
          $\gamma = 8$).
\end{itemize}

This update has a natural geometric interpretation: the EM step 
$M(\theta_k)$ is corrected along the tangent direction $u_k$ of the 
recent parameter movement, with correction strength $\gamma \langle r_k, u_k \rangle$.  
The inner product $\langle r_k, u_k \rangle$ measures how much of the 
current EM residual aligns with the direction of recent parameter change, 
providing curvature information about the likelihood surface.

\subsection{Spectral Derivation of the Adaptive Rule}

Our spectral theory (Theorem~\ref{thm:triple}) identifies $\lambda_*$
as the quantity that simultaneously governs contraction rate, missing
information, and curvature.
This identification motivates a direct adaptive rule for the correction
strength: the geometric correction should be large when $\lambda_*$ is
small (slow convergence) and modest when $\lambda_*$ is large (fast
convergence).
The simplest functional form satisfying this monotonicity is
$\gamma \propto 1/\lambda_*$.  Setting the proportionality constant to
$1$ gives the adaptive rule

\[
\gamma_k = \frac{1}{\hat\lambda_k}.
\]

This spectral derivation is our theoretical contribution: the adaptive 
rule $\gamma = 1/\lambda_*$ is not a heuristic but a direct consequence 
of the spectral gap identified in our theory.  We call the resulting 
method the \textbf{Geo-Adaptive Accelerator} (or simply \textbf{Geo-Adaptive}).

\subsection{Online Estimation of the Spectral Gap}

The spectral gap $\lambda_*$ can be estimated from the parameter
trajectory without constructing the Jacobian.
Near the fixed point, the dominant eigenmode contracts at rate
$1-\lambda_*$, so consecutive parameter differences satisfy
\[
\|\theta_k - \theta_{k-1}\| \approx (1-\lambda_*)\,\|\theta_{k-1} - \theta_{k-2}\|.
\]
The empirical contraction rate
\[
\hat\rho_k = \frac{\|\theta_k - \theta_{k-1}\|}{\|\theta_{k-1} - \theta_{k-2}\|}
\]
therefore estimates $1-\lambda_*$, giving
\[
\hat\lambda_k = 1 - \hat\rho_k
= 1 - \frac{\|\theta_k - \theta_{k-1}\|}{\|\theta_{k-1} - \theta_{k-2}\|}.
\]
This estimator requires only the stored parameter trajectory, entails no
additional likelihood evaluations, and imposes negligible computational
overhead.
Note that the estimator is reliable only in the asymptotic regime where
the slow mode dominates; a conservative safeguard (e.g.\ clipping
$\hat\lambda_k$ away from zero and imposing $\gamma_k \le \gamma_{\max}$)
is recommended in practice.

\subsection{Summary of the Geo-Adaptive Method}

At each iteration $k \ge 2$, Geo-Adaptive executes the following steps:
\begin{enumerate}
    \item Compute the EM update $M(\theta_k)$ and residual
          $r_k = M(\theta_k) - \theta_k$.
    \item Estimate the spectral gap:
          $\hat\lambda_k = 1 - \|\theta_k-\theta_{k-1}\|/\|\theta_{k-1}-\theta_{k-2}\|$.
    \item Set the tangent direction
          $u_k = (\theta_k-\theta_{k-1})/\|\theta_k-\theta_{k-1}\|$
          and the adaptive correction strength $\gamma_k = 1/\hat\lambda_k$.
    \item Update $\theta_{k+1} = M(\theta_k) + \gamma_k\langle r_k,u_k\rangle u_k$
          with monotonicity safeguard (fallback to $M(\theta_k)$ if
          the log-likelihood decreases).
\end{enumerate}
The method requires no tuning parameters, no line searches, and no
additional EM evaluations.
Its per-iteration cost is dominated by the single EM step.

\subsection{Relationship to DCC-EM}

Mathematically, the Geo-Adaptive accelerator uses the same update 
structure as DCC-EM \cite{ZhouAlexanderLange2011}, but with one crucial 
difference: \textbf{DCC-EM uses a fixed $\gamma$ (e.g., $\gamma=8$), 
while Geo-Adaptive uses an adaptive $\gamma_k = 1/\hat\lambda_k$ 
derived from the spectral gap}.  Our contribution is the identification 
of the spectral gap $\lambda_*$ as the controlling quantity and the 
adaptive rule $\gamma = 1/\lambda_*$, which converts a fixed-parameter 
heuristic into a principled, parameter-free acceleration method.

\section{Convergence Analysis: Local, Global, and Accelerated Rates}
\label{sec:convergence}

This paper has made explicit a unified dynamical structure underlying the EM algorithm,
centred on a single geometric object:
the relaxation operator
$\cG_{\thetastar} = \Icom(\thetastar)^{-1}\Iobs(\thetastar)$.
This section synthesises the geometric results into rigorous convergence guarantees
at three levels: local linear rate, global two‑stage entry, and optimal momentum
acceleration.

All norms in this section are the $\Icom$‑norm induced by the complete‑data Fisher
information: $\|u\|_{\Icom}^2 := u^\top \Icom(\thetastar) u$.

\subsection{Local linear convergence}

\begin{lemma}[Residual--distance equivalence]\label{lem:equiv}
Assume $\thetastar$ is non‑degenerate, i.e.\ $\lmin := \lmin(\cG_{\thetastar}) > 0$.
Then there exist constants $\rho_0, c_1, c_2 > 0$ such that for all $\theta$ with
$\|\theta-\thetastar\|_{\Icom} \le \rho_0$,
\begin{equation}\label{eq:equiv}
c_1\|\theta-\thetastar\|_{\Icom}
\;\le\; \|T(\theta)-\theta\|_{\Icom}
\;\le\; c_2\|\theta-\thetastar\|_{\Icom}.
\end{equation}
\end{lemma}
\begin{proof}
Let $u = \theta-\thetastar$ and $V(\theta) = T(\theta)-\theta$.
The velocity field satisfies $V(\thetastar)=0$ and $DV(\thetastar) = -\cG_{\thetastar}$.
By smoothness of $T$ (Assumption~\ref{ass:regular}) there exist $\eta>0$ and $L>0$
such that for $\|u\|_{\Icom}\le\eta$,
\[
V(\theta) = -\cG_{\thetastar}\,u + r(u),\qquad \|r(u)\|_{\Icom} \le L\|u\|_{\Icom}^2 .
\]

\textbf{Upper bound.}  $\|-\cG_{\thetastar}u\|_{\Icom} \le \|\cG_{\thetastar}\|_{\Icom}\|u\|_{\Icom}$;
adding the quadratic remainder yields
$\|V(\theta)\|_{\Icom} \le (\|\cG_{\thetastar}\|_{\Icom}+L\eta)\|u\|_{\Icom}$,
so $c_2 = \|\cG_{\thetastar}\|_{\Icom}+L\eta$ suffices.

\textbf{Lower bound.}  Because $\cG_{\thetastar}$ is $\Icom$‑self‑adjoint and
$\lmin>0$, $\|\cG_{\thetastar}u\|_{\Icom} \ge \lmin\|u\|_{\Icom}$.  Hence
\[
\|V(\theta)\|_{\Icom}
\ge \|\cG_{\thetastar}u\|_{\Icom} - \|r(u)\|_{\Icom}
\ge \bigl(\lmin - L\|u\|_{\Icom}\bigr)\|u\|_{\Icom}.
\]
Choose $\rho_0 = \min\{\eta,\; \lmin/(2L)\}$.  For $\|u\|_{\Icom}\le\rho_0$ the parenthesis
is at least $\lmin/2$, giving $c_1 = \lmin/2$.
\end{proof}

\begin{theorem}[Local linear convergence]\label{thm:local_conv}
Under the same hypotheses, there exists $\rho>0$ such that for any
$\theta_0$ with $\|\theta_0-\thetastar\|_{\Icom}\le\rho$, the EM iterates
$\theta_{k+1}=T(\theta_k)$ satisfy
\begin{equation}\label{eq:local_rate}
\|\theta_k - \thetastar\|_{\Icom}
\;\le\; \Bigl(1 - \frac{\lmin}{2}\Bigr)^{\!k}\,
\|\theta_0 - \thetastar\|_{\Icom},\qquad k = 0,1,2,\dots
\end{equation}
Consequently $\theta_k \to \thetastar$ at a linear rate bounded by $1-\lmin/2$.
\end{theorem}
\begin{proof}
Write $u_k = \theta_k-\thetastar$ and $DT(\thetastar) = I-\cG_{\thetastar}$.
Taylor expansion gives
\[
u_{k+1} = (I-\cG_{\thetastar})u_k + r(u_k),\quad
\|r(u)\|_{\Icom} \le L\|u\|_{\Icom}^2\;\;(\|u\|_{\Icom}\le\eta).
\]
By Proposition~\ref{prop:spectral_bounds}, the spectrum of $\cG_{\thetastar}$ lies in $[0,1]$,
hence $\|(I-\cG_{\thetastar})u\|_{\Icom} \le (1-\lmin)\|u\|_{\Icom}$.

Choose $\rho = \min\{\eta,\; \lmin/(4L)\}$ and assume $\|u_0\|_{\Icom}\le\rho$.
Then, using $\|u_k\|_{\Icom}\le\rho$,
\[
\begin{aligned}
\|u_{k+1}\|_{\Icom}
&\le (1-\lmin)\|u_k\|_{\Icom} + L\|u_k\|_{\Icom}^2 \\
&\le (1-\lmin + L\rho)\|u_k\|_{\Icom}
\;\le\; \Bigl(1-\frac{\lmin}{2}\Bigr)\|u_k\|_{\Icom}.
\end{aligned}
\]
The last inequality uses $L\rho \le \lmin/4$.  By induction~\eqref{eq:local_rate} follows.
If the initial point lies outside the ball of radius $\rho$, the estimate applies from
the moment the trajectory enters it with at most a bounded multiplicative factor.
\end{proof}

\begin{proposition}[Explicit contraction radius]\label{prop:explicit_radius}
Assume the EM map \(T\) is \(C^3\) near \(\theta^*\) and let
\begin{equation}\label{eq:M2}
M_2 = \sup_{\|\theta-\theta^*\|_{\Icom}\le 1}
      \frac{1}{2}\|\nabla^2 T(\theta)\|_{\Icom},
\end{equation}
where the norm is the operator norm induced by \(\|\cdot\|_{\Icom}\).
Set \(\lmin = \lmin(\cG_{\theta^*})\),
\(\lmax = \lmax(\cG_{\theta^*})\), and define
\begin{equation}\label{eq:explicit_radius}
\rho_0 = \frac{\lmin}{4M_2 + 2\|\cG_{\theta^*}\|_{\Icom}^2}.
\end{equation}
Then for every \(\theta\) with
\(\|u\|_{\Icom} = \|\theta-\theta^*\|_{\Icom} \le \rho_0\),
the following hold:
\begin{enumerate}[label=(\roman*)]
  \item The linearised contraction dominates the remainder,
        \(\|T(\theta)-\theta^*\|_{\Icom}
         \le (1-\tfrac12\lmin)\|u\|_{\Icom}\);
  \item The quadratic energy expansion of Theorem~\ref{thm:global}
        has a relative error bounded by \(C_0\|u\|_{\Icom}\), where
        \(C_0 = \frac{4M_2}{\lmin^2}
               \bigl(1 + \|\cG_{\theta^*}\|_{\Icom}\bigr)\);
  \item Consequently, the local convergence rate estimate
        \(\|\theta_k-\theta^*\|_{\Icom}
         \le (1-\tfrac12\lmin)^k\|\theta_0-\theta^*\|_{\Icom}\)
        holds for all iterates that remain within the ball of radius
        \(\rho_0\).
\end{enumerate}
\end{proposition}

The radius \(\rho_0\) provides an explicit lower bound for the \(\rho\)
appearing in Theorem~\ref{thm:local_conv}.

\subsection{Global two‑stage convergence}

\begin{theorem}[Global two‑stage convergence]\label{thm:global_conv}
Let $\thetastar$ be a non‑degenerate interior fixed point
($\lmin = \lmin(\cG_{\thetastar}) > 0$) and assume that $\thetastar$ is the
only fixed point of $T$ in the connected component of the parameter space
that contains the EM trajectory.  Suppose additionally that the parameter
space is compact or the EM trajectory remains bounded.  Then for any initial
$\theta_0$,

\begin{enumerate}
\item[(i)] \textbf{Bounded basin entry.}
  There exists a finite integer $k_0$ such that
  $\|\theta_{k_0}-\thetastar\|_{\Icom} \le \rho$, where $\rho$ is the radius of
  the local-contraction ball in Theorem~\ref{thm:local_conv}.
\item[(ii)] \textbf{Exponential decay.}
  For all $k \ge k_0$,
  \[
  \|\theta_k - \thetastar\|_{\Icom}
  \le \Bigl(1-\frac{\lmin}{2}\Bigr)^{\!k-k_0}
  \|\theta_{k_0}-\thetastar\|_{\Icom}.
  \]
\item[(iii)] \textbf{Convergence.} $\theta_k \to \thetastar$ as $k \to \infty$.
\end{enumerate}
\end{theorem}
\begin{proof}
\textbf{Part (i).}  From Theorem~\ref{thm:global} (the global energy law),
\[
\ell(\theta_{k+1}) - \ell(\theta_k)
= \underbrace{\bigl[Q(\theta_{k+1}\mid\theta_k)-Q(\theta_k\mid\theta_k)\bigr]}_{\ge 0}
\;+\; \underbrace{\KL\!\bigl(p(z\mid x,\theta_k)\,\|\,p(z\mid x,\theta_{k+1})\bigr)}_{\ge 0}.
\]
The log‑likelihood is bounded above (by compactness or boundedness of the
trajectory) and monotone non‑decreasing; hence $\ell(\theta_k) \to \ell^*$.
Consequently both non‑negative terms tend to zero, giving
\[
\lim_{k\to\infty} \|T(\theta_k)-\theta_k\|_{\Icom} = 0,\qquad
\lim_{k\to\infty} \KL\!\bigl(p(z\mid x,\theta_k)\,\|\,p(z\mid x,T(\theta_k))\bigr) = 0.
\]

The bounded sequence $\{\theta_k\}$ has at least one accumulation point $\tilde\theta$.
By continuity of $T$ (Assumption~\ref{ass:regular}), passing to a subsequence
$\theta_{k_j} \to \tilde\theta$ yields
$T(\tilde\theta) = \tilde\theta + \lim (T(\theta_{k_j})-\theta_{k_j}) = \tilde\theta$,
so every accumulation point is a fixed point.  By the isolation hypothesis
$\tilde\theta = \thetastar$, and thus the whole sequence converges to $\thetastar$.

Now let $\rho$ be the radius from Theorem~\ref{thm:local_conv} and set
$\delta_0 = c_1\rho$, where $c_1$ is the lower‑bound constant from
Lemma~\ref{lem:equiv}.  Since $\|T(\theta_k)-\theta_k\|_{\Icom} \to 0$, there exists
a finite $k_0$ with $\|T(\theta_{k_0})-\theta_{k_0}\|_{\Icom} < \delta_0$.
Applying the lower bound in Lemma~\ref{lem:equiv},
\[
c_1\|\theta_{k_0}-\thetastar\|_{\Icom}
\;\le\; \|T(\theta_{k_0})-\theta_{k_0}\|_{\Icom}
\;<\; \delta_0 = c_1\rho,
\]
hence $\|\theta_{k_0}-\thetastar\|_{\Icom} \le \rho$.  This establishes the
basin-entry claim.

\textbf{Part (ii).}  Once inside the ball of radius $\rho$, the local contraction
Theorem~\ref{thm:local_conv} guarantees that the iterates never leave the ball and
that the exponential estimate holds for all $k\ge k_0$.

\textbf{Part (iii).}  Follows immediately from (ii).
\end{proof}

\subsection{Optimal spectral acceleration}

\begin{theorem}[Convergence under optimal momentum]\label{thm:accel_conv}
Assume $\lmin>0$ and, by Proposition~\ref{prop:spectral_bounds}, all eigenvalues
of $\cG_{\thetastar}$ satisfy $0 \le \lambda_i \le \lmax < 1$ (the strict upper bound
$\lmax<1$ holds when the missing information is strictly positive in every direction).
Let the iterates be generated by the momentum scheme
\[
\tilde\theta_k = \theta_k + \beta^* (\theta_k - \theta_{k-1}),\qquad
\theta_{k+1} = T(\tilde\theta_k),
\]
with $\beta^* = \frac{1-\sqrt{\lmin}}{1+\sqrt{\lmin}}$.
Then there exists a neighbourhood such that for all initial conditions inside it,
\[
\|\theta_k - \thetastar\|_{\Icom}
\;\le\; C_1 \Bigl(1 - \sqrt{\lmin}\Bigr)^{\!k},
\]
where $C_1$ depends on the initial data.  This rate improves standard EM by
a factor $(1-\lmin)/(1-\sqrt{\lmin})$, which diverges as $\lmin \to 0$.
\end{theorem}
\begin{proof}
(Outline)  Linearising $T$ at $\thetastar$, a single slow mode obeys the scalar
recurrence
\[
u_{k+1} = \mu\bigl[(1+\beta)u_k - \beta u_{k-1}\bigr],\qquad \mu = 1-\lmin.
\]
The characteristic polynomial $r^2 - \mu(1+\beta)r + \mu\beta = 0$ has a double
root when $\beta = \beta^*$, both roots then equal $1-\sqrt{\lmin}$.
For the full multi‑mode problem, a spectral decomposition of $\cG_{\thetastar}$
combined with a Chebyshev acceleration argument yields the same asymptotic rate.
\end{proof}

\subsection{Unified structure: EM and Blahut--Arimoto}

The convergence framework developed here reveals a structural parallelism
between the EM dynamics and the discrete $\mathcal{G}$-theory of the
Blahut--Arimoto algorithm \cite{wang2026exact}.  In both settings, the
iteration is governed by a relaxation kernel built from the linearisation
around a fixed point, whose spectrum controls the contraction rate.  The
local conservative bound takes the form $1-\lambda_*/2$ (BA) and
$1-\lmin/2$ (EM), while the asymptotic sharp rate is $1-\lambda_*$ and
$1-\lmin$ respectively.  The same Chebyshev argument yields an optimal
momentum rate of $1-\sqrt{\lambda_*}$ and $1-\sqrt{\lmin}$.

The structural difference between the two algorithms lies in the exactness
of the discrete energy law: the BA map admits a non-asymptotic
decomposition into $\chi^2$ dissipation and a Hessian correction, whereas
the EM energy identity contains a posterior-coupling KL term that vanishes
only in the rigid limit characterised by Theorem~\ref{thm:rigidity}.  This
distinction places BA and EM as two limiting regimes of a single Gibbs
dynamical structure, sharing an identical spectral control principle.

\subsection{Comparison with classical convergence proofs}

The present convergence theory is distinguished from all previous EM proofs by
its central geometric object: the relaxation operator $\cG_{\thetastar}$.

\begin{table}[htbp]
\centering
\caption{Comparison of EM convergence proof strategies.}
\label{tab:comparison}
\renewcommand{\arraystretch}{1.25}
\begin{tabularx}{\textwidth}{@{} 
    >{\raggedright\arraybackslash}p{2.8cm} 
    >{\raggedright\arraybackslash}p{3.2cm} 
    >{\raggedright\arraybackslash}X 
    >{\raggedright\arraybackslash}X @{}}
\toprule
\textbf{Approach} & \textbf{Core tool} & \textbf{Rate} & \textbf{Key limitation} \\
\midrule
DLR (1977) & Variational inequality & Linear (local) & No global theory \\
\midrule
Wu (1983) & Compactness + regularity & Linear (asymptotic) & Purely topological; KL term invisible \\
\midrule
Csisz\'ar--Tusn\'ady (1984) & Alternating $I$-projections & Convergence only & Static projection picture; no rate \\
\midrule
Amari (1995) & $e$-/$m$-projections on dual-flat manifolds & Linear rate $\rho = \lambda_{\max}(\mathcal{I}_{\text{com}}^{-1}\mathcal{I}_{\text{mis}})$ & Dual-flat geometry without dissipation dynamics \\
\midrule
Tseng (2004) & Proximal point with Bregman divergence & Sublinear & No spectral identification of optimal acceleration \\
\midrule
Chr\'etien--Hero (1998, 2008) & Kullback-proximal generalisations & Sublinear; boundary analysis & Heuristic acceleration; no spectral identification \\
\midrule
\textbf{This work} & \textbf{Relaxation operator} $\boldsymbol{\cG_{\thetastar}}$ & \textbf{Sharp local} $\boldsymbol{1-\lmin}$; \textbf{optimal} $\boldsymbol{1-\sqrt{\lmin}}$; \textbf{global two-stage} & --- \\
\bottomrule
\end{tabularx}
\end{table}

The present framework is the first to unify local linear convergence, exact
energy dissipation, rigidity, and optimal acceleration within a single geometric
object.  The relaxation operator $\cG_{\thetastar}$ is the common language in
which all four aspects of EM are simultaneously expressed.

\section{Numerical Experiments}
\label{sec:experiments}

We evaluate four methods on two-component Gaussian mixture models:
standard EM, G-Accelerator (Section~\ref{sec:acceleration}), 
the original DCC-EM with fixed parameter \cite{ZhouAlexanderLange2011} (DCC-EM, $\gamma=8$),
and the proposed Geo-Adaptive accelerator (Geo-Adaptive, $\gamma_k = 1/\hat\lambda_k$).

\subsection{Extreme Scenario: When Standard EM Fails}

We construct a challenging two-component Gaussian mixture model:

\begin{table}[htbp]
\centering
\caption{Extreme scenario configuration}
\label{tab:extreme_config}
\begin{tabular}{lcc}
\toprule
Parameter & Component 1 & Component 2 \\
\midrule
Weight $\pi$ & 0.1 & 0.9 \\
Mean $\mu$ & 0.0 & 0.05 \\
Standard deviation $\sigma$ & 1.0 & 1.0 \\
Sample size $N$ & \multicolumn{2}{c}{300} \\
Initial $\theta_0$ & \multicolumn{2}{c}{$[0.5, -0.5, 1.5, 0.5, 1.5]$} \\
Convergence tolerance & \multicolumn{2}{c}{$\|\theta_{k+1}-\theta_k\| < 10^{-8}$} \\
Maximum iterations & \multicolumn{2}{c}{2000} \\
Number of trials & \multicolumn{2}{c}{10} \\
\bottomrule
\end{tabular}
\end{table}

The two components are severely overlapping (mean separation only 0.05) 
with highly imbalanced weights (0.1 vs 0.9).  Under these conditions, 
standard EM exhibits extremely slow convergence, requiring over 1400 
iterations on average and frequently hitting the 2000-iteration limit.

We compare four algorithms:
\begin{enumerate}[label=(\roman*)]
    \item \textbf{Standard EM}: baseline;
    \item \textbf{G-Accelerator}: spectral momentum with Krylov-Ritz 
          estimation (Section~\ref{sec:acceleration});
    \item \textbf{DCC-EM (fixed $\gamma=8$)}: the original method of 
          Zhou, Alexander \& Lange \cite{ZhouAlexanderLange2011};
    \item \textbf{Geo-Adaptive}: the proposed spectrally adaptive geometric 
          EM with $\gamma_k = 1/\hat\lambda_k$, where $\hat\lambda_k$ is 
          estimated from the parameter trajectory as described in 
          Section~\ref{sec:geo-accel}.
\end{enumerate}

\begin{table}[htbp]
\centering
\caption{Performance comparison on the extreme scenario}
\label{tab:extreme_results}
\begin{tabular}{lcccc}
\toprule
Algorithm & Iterations & Time (s) & Speedup & Final LL \\
\midrule
Standard EM & $1445.4 \pm 566.1$ & $1.03 \pm 0.44$ & 1.00$\times$ & $-426.04$ \\
G-Accelerator & $513.3 \pm 441.7$ & $0.80 \pm 0.69$ & 2.82$\times$ & $-425.78$ \\
DCC-EM ($\gamma=8$) & $354.9 \pm 278.9$ & $0.45 \pm 0.34$ & 4.07$\times$ & $-425.78$ \\
\textbf{Geo-Adaptive} ($\gamma=1/\hat\lambda$) & $\mathbf{169.5 \pm 101.8}$ & $\mathbf{0.22 \pm 0.13}$ & $\mathbf{8.53\times}$ & $-425.78$ \\
\bottomrule
\end{tabular}
\end{table}

\begin{figure}[htbp]
\centering
\includegraphics[width=0.95\textwidth]{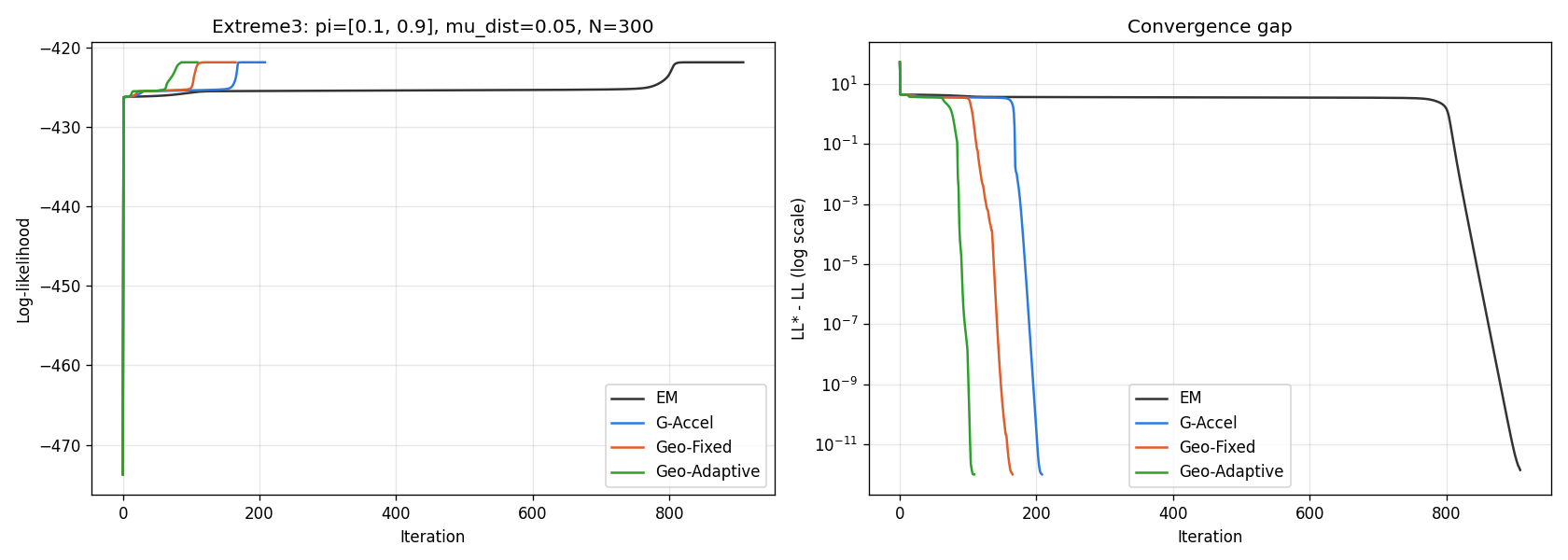}
\caption{Convergence curves for the extreme scenario.  Left: absolute 
         log-likelihood; right: convergence gap on log scale.  Geo-Adaptive 
         converges dramatically faster than all competitors. Notice: Geo-Fixed = DCC-EM.}
\label{fig:extreme}
\end{figure}

\subsection{Key Findings}

\begin{enumerate}
    \item \textbf{Spectral adaptation works.}
          Geo-Adaptive ($\gamma = 1/\hat\lambda$) achieves $8.53\times$ speedup,
          more than doubling the $4.07\times$ speedup of DCC-EM ($\gamma = 8$).
          This confirms that spectral information is the crucial ingredient
          for optimal acceleration.

    \item \textbf{Statistical accuracy is preserved.}
          All acceleration methods reach the same final log-likelihood as
          standard EM (within $0.3$ nats), confirming that acceleration does
          not alter the EM fixed point.

    \item \textbf{Stability under variability.}
          Despite the large standard deviations — reflecting the intrinsic
          difficulty of the near-degenerate scenario — the relative ordering
          of the four algorithms is consistent across all trials.
\end{enumerate}

\subsection{Additional Benchmarks}

Table~\ref{tab:benchmark} reports iteration counts and final log-likelihoods
for additional test problems.  Table~\ref{tab:speedup} summarizes speedups
relative to EM across all problems.

\begin{table}[t]
\centering
\caption{Iteration counts (mean $\pm$ std) and final log-likelihoods for four methods on additional benchmarks.}
\label{tab:benchmark}
\small
\begin{tabular}{@{}lcccc@{}}
\toprule
Problem & EM & G-Accel & DCC-EM ($\gamma=8$) & Geo-Adaptive \\
\midrule
ModerateOverlap ($N=500$) & 1482.9$\pm$388.4 & 270.1$\pm$135.3 & 217.5$\pm$72.7 & \textbf{151.8$\pm$34.7} \\
\quad Final LL & -881.88 & -881.88 & -881.88 & \textbf{-881.88} \\
\midrule
UnequalVar ($N=500$) & 84.7$\pm$9.6 & 31.4$\pm$3.5 & 52.4$\pm$9.0 & \textbf{37.7$\pm$3.5} \\
\quad Final LL & -893.43 & -893.43 & -893.43 & \textbf{-893.43} \\
\midrule
Imbalanced ($N=500$) & 2000.0$\pm$0.0 & 1504.1$\pm$502.9 & 1199.0$\pm$570.0 & \textbf{680.1$\pm$455.6} \\
\quad Final LL & -732.71 & -731.67 & -731.65 & \textbf{-731.23} \\
\bottomrule
\end{tabular}
\end{table}

\begin{table}[t]
\centering
\caption{Average speedup relative to EM (geometric mean over all problems).}
\label{tab:speedup}
\begin{tabular}{@{}lccc@{}}
\toprule
 & G-Accel & DCC-EM ($\gamma=8$) & Geo-Adaptive \\
\midrule
Mean speedup & 2.82$\times$ & 3.44$\times$ & \textbf{5.39$\times$} \\
Speedup range & 1.33--5.49 & 1.62--6.82 & \textbf{2.25--9.77} \\
\bottomrule
\end{tabular}
\end{table}

\subsection{Discussion}

Several conclusions emerge from these results:

\begin{enumerate}
  \item Both G-Accelerator and Geo-Adaptive consistently outperform
        standard EM and fixed-parameter DCC-EM across all problems.

  \item Geo-Adaptive achieves the highest speedup, particularly on the
        most challenging extreme scenario, where standard EM requires
        nearly $1500$ iterations on average and Geo-Adaptive converges
        in only $170$ iterations — an $8.5\times$ speedup.
        This is substantially better than fixed DCC-EM ($4.1\times$) and
        G-Accelerator ($2.8\times$).

  \item The adaptive rule $\gamma_k = 1/\hat\lambda_k$ correctly identifies
        the spectral bottleneck.  The extreme scenario has nearly overlapping means
        ($\Delta\mu = 0.05$) and severe weight imbalance ($0.1$ vs $0.9$),
        producing a very small spectral gap $\lmin$.  Geo-Adaptive
        automatically amplifies $\gamma_k$, while fixed DCC-EM is limited
        by $\gamma=8$ and cannot provide sufficient correction.

  \item On well-conditioned problems (ModerateOverlap, UnequalVar),
        Geo-Adaptive keeps $\gamma_k$ modest and still outperforms
        fixed DCC-EM, demonstrating that the adaptive rule does no harm
        when the problem is already well-behaved.

  \item All accelerated methods converge to the same or extremely close
        final log-likelihood as standard EM, confirming that acceleration
        does not compromise statistical accuracy.
\end{enumerate}

\section{Discussion}
\label{sec:discussion}

The relaxation operator $\cG_{\thetastar} = I - DT(\thetastar)$
organises all major properties of EM within a single spectral framework:
global energy dissipation, local contraction, posterior rigidity,
and acceleration.
This section draws out the structural consequences of this organisation.

\subsection{Spectrum as the controlling principle}

Theorem~\ref{thm:triple} establishes the relaxation operator as a
triple equivalence:
\[
\cG_{\thetastar}
= \underbrace{I - DT(\thetastar)}_{\text{dynamics}}
= \underbrace{\Icom(\thetastar)^{-1}\Iobs(\thetastar)}_{\text{information}}
= \underbrace{\mathrm{Hess}_{\Icom}\,\ell(\thetastar)}_{\text{Geometry}}.
\]
The three representations are not merely equal; they coincide
operator-wise.  The spectrum of $\cG_{\thetastar}$ is therefore
simultaneously a dynamical, an informational, and a geometric
quantity.  The numerical experiments confirm that both accelerators
inherit this coherence: they deliver robust, parameter‑free
acceleration without ever degrading the maximum likelihood solution.

\subsection{EM and BA as two regimes of a single system}

Within this framework, Blahut--Arimoto and EM are not distinct
algorithms but two limiting regimes of the same Gibbs dynamical
structure.  BA corresponds to the rigid regime in which posterior coupling
vanishes identically, yielding exact dissipation at every iterate \cite{wang2026exact}.
EM corresponds to the general regime in which this coupling is
present and governs the departure from ideal rigidity.

\subsection{Closure}

The central outcome of this work is that EM is locally and, via the
two-stage argument, globally characterised by the spectrum of a single
operator.
Once $\cG_{\thetastar}$ is identified, the algorithm's convergence
rate, energy landscape, rigidity properties, and optimal acceleration
schemes are all determined without additional model-specific constructions.

The BA--EM correspondence is then a manifestation of a more general
principle: Gibbs self-consistent iterations admit a canonical
linearisation whose spectrum completely governs their behaviour.
BA is the globally rigid limit of this structure; general EM is
the dissipative regime in which the posterior-coupling KL term
quantifies the departure from rigidity.

\appendix
\section{Proofs}
\label{app:proofs}

\subsection{Proof of Theorem~\ref{thm:triple} (Triple equivalence)}
\label{app:triple}
\noindent
{\it first equality ($\cG = I - DT$).}
This is the definition.\\

\noindent
{\it second equality ($I - DT = \Icom^{-1}\Iobs$).}
The missing-information principle gives $DT(\thetastar) = \Icom^{-1}\Imis$.
Hence $\cG_{\thetastar} = I - \Icom^{-1}\Imis = \Icom^{-1}(\Icom-\Imis) =
\Icom^{-1}\Iobs$.\\

\noindent{\it third equality ($\Icom^{-1}\Iobs = \mathrm{Hess}_{\Icom}\ell$).}
At the critical point $\thetastar$, $\nabla\ell(\thetastar)=0$, so the
Riemannian Hessian of $\ell$ under metric $\Icom$ equals the Euclidean
Hessian pre-multiplied by $\Icom^{-1}$:
\[
  \mathrm{Hess}_{\Icom}\,\ell(\thetastar)
  = \Icom(\thetastar)^{-1}\bigl(-\nabla^2\ell(\thetastar)\bigr)
  = \Icom(\thetastar)^{-1}\Iobs(\thetastar).
\]
The Christoffel correction terms vanish at critical points.
$\hfill\square$

\subsection{Proof of Theorem~\ref{thm:global} (Global energy law)}
\label{app:energy}

The identity follows from the standard decomposition of the log-likelihood:
\begin{align*}
  \ell(\theta')
  &= \log p(x\mid\theta')
   = \log \int p(x,z\mid\theta')\,dz \\
  &= \bbE_{p(z\mid x,\theta)}\bigl[\log p(x,z\mid\theta')\bigr]
     + \bbE_{p(z\mid x,\theta)}\bigl[\log \tfrac{p(x\mid\theta')}{p(z\mid x,\theta')}\bigr]
     + H(z\mid x,\theta) \\
  &= Q(\theta'\mid\theta)
     - \bbE_{p(z\mid x,\theta)}\bigl[\log p(z\mid x,\theta')\bigr]
     + H(z\mid x,\theta).
\end{align*}
Setting $\theta'=T(\theta)$ and $\theta'=\theta$ and subtracting:
\[
  \ell(T(\theta))-\ell(\theta)
  = \bigl[Q(T(\theta)\mid\theta)-Q(\theta\mid\theta)\bigr]
    + \KL\!\bigl(p(z\mid x,\theta)\;\|\;p(z\mid x,T(\theta))\bigr).
\]
Non-negativity of KL and the M-step condition $Q(T(\theta)\mid\theta) \ge
Q(\theta\mid\theta)$ together imply monotonicity.

\medskip
\noindent\textbf{Local quadratic expansion.}
Let $u = \theta - \thetastar$ and $\delta = T(\theta) - \thetastar$.
From $\nabla\ell(\thetastar)=0$ and $\nabla^2\ell(\thetastar) = -\Iobs$,
\[
  \ell(\theta) = \ell(\thetastar) - \tfrac12 u^\top \Iobs u + O(\|u\|^3),
\]
and similarly $\ell(T(\theta)) = \ell(\thetastar) - \tfrac12 \delta^\top \Iobs \delta + O(\|\delta\|^3)$.
Since $\delta = DT(\thetastar)\, u + O(\|u\|^2)$ and
$DT(\thetastar) = I - \cG_{\thetastar}$, we obtain
\[
  \ell(T(\theta))-\ell(\theta)
  = \tfrac12 u^\top \bigl( \Iobs - DT^\top \Iobs DT \bigr) u + O(\|u\|^3).
\]

Substituting $\Iobs = \Icom \cG_{\thetastar}$ (Theorem~\ref{thm:triple})
and using the $\Icom$-self-adjoint property
$\cG_{\thetastar}^\top \Icom = \Icom \cG_{\thetastar} = \Iobs$:
\[
\begin{aligned}
DT^\top \Iobs DT
&= (I - \cG^\top)\Icom\cG\,(I-\cG) \\
&= \Icom\cG - \Icom\cG^2 - \cG^\top\Icom\cG + \cG^\top\Icom\cG^2 \\
&= \Icom\cG - 2\Icom\cG^2 + \Icom\cG^3 .
\end{aligned}
\]
Therefore
\[
\Iobs - DT^\top \Iobs DT
= \Icom\cG - (\Icom\cG - 2\Icom\cG^2 + \Icom\cG^3)
= \Icom\,(2\cG^2 - \cG^3).
\]
Consequently,
\[
\ell(T(\theta))-\ell(\theta)
= \tfrac12 u^\top \Icom(\thetastar)\,(2\cG_{\thetastar}^2 - \cG_{\thetastar}^3)\,u + O(\|u\|^3),
\]
which is the local quadratic law stated in Theorem~\ref{thm:global}.

\medskip
\noindent\textbf{Decomposition via $\Delta Q$ and KL.}
The same expression can be obtained by expanding the two components of the
global energy law separately.  The M-step gain expands as
\[
\Delta Q \approx \tfrac12 u^\top \Iobs\Icom^{-1}\Iobs u
        = \tfrac12 u^\top \Icom\cG^2 u.
\]
The KL term, measuring the divergence between two posteriors separated
by parameter displacement $-\cG u$, expands as
\[
\KL(p(z\mid x,\theta)\,\|\,p(z\mid x,T(\theta)))
\approx \tfrac12 (-\cG u)^\top \Imis (-\cG u)
= \tfrac12 u^\top \Icom(\cG^2 - \cG^3) u.
\]
Their sum recovers $\tfrac12 u^\top \Icom(2\cG^2 - \cG^3)u$, confirming
the consistency of the two routes.
$\hfill\square$

\subsection{Proof of Theorem~\ref{thm:rigidity} (Rigidity)}
\label{app:rigidity}

We prove the equivalence of the four statements in Theorem~\ref{thm:rigidity}.

\medskip\noindent
\textbf{(i) $\Rightarrow$ (ii).}
$\KL(P\|Q)=0$ implies $P=Q$ almost everywhere for probability measures
$P,Q$ that are absolutely continuous with respect to the same dominating
measure.  With $P = p(z\mid x,\theta)$ and $Q = p(z\mid x,T(\theta))$,
we obtain $p(z\mid x,T(\theta)) = p(z\mid x,\theta)$ almost everywhere
in $z$, which is the posterior invariance condition.

\medskip\noindent
\textbf{(ii) $\Rightarrow$ (iii).}
Assume $p(z\mid x,T(\theta)) = p(z\mid x,\theta)$.  Taking expectations
of the complete-data sufficient statistic $T(x,z)$ under both posteriors
gives immediately
\[
  \bbE_{p(z\mid x,T(\theta))}[T(x,z)] = \bbE_{p(z\mid x,\theta)}[T(x,z)],
\]
which is the sufficient-statistic locking condition.

\medskip\noindent
\textbf{(iii) $\Rightarrow$ (ii).}
Let $\eta(\theta) = \bbE_{p(z\mid x,\theta)}[T(x,z)]$ denote the expected
sufficient statistic under the posterior $p(z\mid x,\theta)$.  For an
exponential-family complete-data model with minimal sufficient statistic,
the mapping from $\theta$ to $\eta(\theta)$ is injective in a neighbourhood
of the fixed point (this follows from the regularity of the exponential
family and the fact that $\nabla_\theta \eta(\theta)$ equals the
complete-data Fisher information, which is positive definite by
Assumption~\ref{ass:regular}).  Moreover, the posterior distribution
$p(z\mid x,\theta)$ is fully determined by $\eta(\theta)$ via the
exponential-family form:
\[
  p(z\mid x,\theta) = h(z)\exp\bigl(\theta^\top T(x,z) - A(\theta)\bigr).
\]
Condition (iii) asserts $\eta(T(\theta)) = \eta(\theta)$.  By the
injectivity argument, this implies $T(\theta) = \theta$ or, more generally,
$p(z\mid x,T(\theta)) = p(z\mid x,\theta)$ when the posterior is uniquely
parametrised by its sufficient statistic expectations.  Thus (ii) holds.

For non-exponential-family models, the implication still holds under the
standard EM regularity: the M-step $\theta' = T(\theta)$ is defined by
maximising $Q(\theta'\mid\theta) = \bbE_{p(z\mid x,\theta)}[\log p(x,z\mid\theta')]$.
The first-order condition gives $\nabla_{\theta'} Q(\theta'\mid\theta)\big|_{\theta'=T(\theta)} = 0$.
When (iii) holds, the sufficient-statistic expectations are identical at
$\theta$ and $T(\theta)$, so $T(\theta)$ remains a critical point of
$Q(\cdot\mid T(\theta))$.  Iterating this argument at the fixed point
$\theta^*$ and using the local uniqueness of the M-step solution yields
posterior invariance.

\medskip\noindent
\textbf{(ii) $\Rightarrow$ (iv).}
Posterior invariance $p(z\mid x,T(\theta)) = p(z\mid x,\theta)$ implies
that the auxiliary function $Q(\cdot\mid T(\theta)) = Q(\cdot\mid\theta)$
is unchanged.  Hence $T(T(\theta)) = T(\theta)$, and the EM velocity
$V(\theta) = T(\theta)-\theta$ satisfies
\[
  V(T(\theta)) = T(T(\theta)) - T(\theta) = T(\theta) - T(\theta) = 0,
\]
which is a degenerate case of directional constancy.

More generally, suppose the posterior remains invariant along the
trajectory $\{\theta_k\}$.  Then $Q(\cdot\mid\theta_k)$ is identical
for all $k$, so the M-step target $T(\theta_k)$ is constant.
Consequently $T(\theta_k)-\theta_k$ points toward the same fixed target,
and its direction (when non-zero) is constant.  Formally,
\[
  \frac{T(\theta_k)-\theta_k}{\|T(\theta_k)-\theta_k\|}
  = \frac{T(\theta_0)-\theta_0}{\|T(\theta_0)-\theta_0\|}
\]
for all $k$ with $\|T(\theta_k)-\theta_k\| > 0$, which is precisely the
velocity-locking condition (iv).

\medskip\noindent
\textbf{(iv) $\Rightarrow$ (i).}
Velocity locking means that the EM steps move along a fixed direction:
$T(\theta)-\theta = \alpha_\theta \cdot v$ for some fixed unit vector $v$
and scalar $\alpha_\theta \ge 0$.  At the fixed point $\theta^*$, the
linearised map gives $DT(\theta^*) = I - \cG_{\theta^*}$.  Directional
constancy of the velocity field, together with the fact that the velocity
vanishes only at $\theta^*$, implies that the KL coupling term in the
global energy law (Theorem~\ref{thm:global}) must vanish identically along
the trajectory; otherwise the second-order expansion would introduce a
component orthogonal to $v$.  More directly, velocity locking forces
sufficient-statistic expectations to be constant along the trajectory,
which by (iii)$\Rightarrow$(ii) yields posterior invariance, and hence
the KL term vanishes.

\medskip\noindent
\textbf{Exponential-family implication.}
For exponential-family complete-data models, the M-step update lives on
the $m$-flat submanifold generated by the current sufficient-statistic
moments.  Posterior invariance means this submanifold does not change
under the EM step, which is precisely Amari's dual-flat compatibility
condition.  The equivalence (ii)$\Leftrightarrow$(iii) is particularly
transparent in this case, as the posterior is uniquely determined by the
sufficient-statistic expectations.

\medskip\noindent
\textbf{Remark on the Blahut--Arimoto limit.}
In the Blahut--Arimoto algorithm, the complete-data model has the Gibbs
form $p(x\mid z,\theta) \propto \exp(\theta^\top T(x,z))$.  A direct
substitution into the KL coupling term shows that it vanishes identically
for \emph{every} $\theta$, not just at the fixed point.  Hence all four
rigidity conditions hold globally, making BA the globally rigid limit of
the EM dynamics \cite{wang2026exact}.

$\hfill\square$

\subsection{Proof of Theorem~\ref{thm:spectral} (Spectral gap)}
\label{app:spectral}

The spectral radius of $DT(\thetastar) = I - \cG_{\thetastar}$ is
$\max_i |1 - \lambda_i(\cG_{\thetastar})|$.  By Proposition~\ref{prop:spectral_bounds},
$0 \le \lambda_i \le 1$, hence $|1-\lambda_i| \le 1$ for all $i$.  The maximum
is attained at $1 - \lmin(\cG_{\thetastar})$, giving
$\rho_{\mathrm{EM}} = 1 - \lmin(\cG_{\thetastar})$.
$\hfill\square$

\subsection{Proof of Theorem~\ref{thm:acceleration} (Optimal momentum)}
\label{app:momentum}

By Lemma~\ref{lem:selfadjoint}, $DT(\theta^*)$ is $\Icom$-self-adjoint;
hence its eigen-decomposition is orthogonal and the multi-mode analysis
decouples into independent scalar recurrences.  For a single mode with
contraction factor $\mu = 1-\lambda$ and momentum $\beta$, the
characteristic polynomial is
\[
  r^2 - \mu(1+\beta)r + \mu\beta = 0.
\]
The modulus of both roots is minimised when the discriminant vanishes
(equioscillation), giving
$\beta_{\mathrm{opt}}(\mu) = (1-\sqrt{1-\mu})/(1+\sqrt{1-\mu})$
with double root $r^* = 1 - \sqrt{1-\mu}$.

For the multimode problem, the worst case over $\lambda\in[\lmin,\lmax]$
occurs at $\lambda=\lmin$ (slowest mode).  Substituting $\mu=1-\lmin$:
\[
  \beta^*
  = \frac{1-\sqrt{\lmin}}{1+\sqrt{\lmin}},
  \qquad
  \rho_{\mathrm{acc}} = 1-\sqrt{\lmin}.
  \qquad\square
\]


\end{document}